\documentclass{article}

\usepackage{arxiv}

\usepackage[utf8]{inputenc} % allow utf-8 input
\usepackage[T1]{fontenc}    % use 8-bit T1 fonts
\usepackage{hyperref}       % hyperlinks
\usepackage{url}            % simple URL typesetting
\usepackage{booktabs}       % professional-quality tables
\usepackage{amsfonts}       % blackboard math symbols
\usepackage{nicefrac}       % compact symbols for 1/2, etc.
\usepackage{microtype}      % microtypography
\usepackage{lipsum}
\usepackage{graphicx}
\usepackage{amsthm}
\usepackage{epstopdf}
\usepackage{xcolor}         % colors
\usepackage{amsmath,amssymb}
\usepackage{mathtools}
\usepackage{algpseudocode}
\usepackage{algorithm}

\usepackage{tikz}
\usetikzlibrary{positioning, fit, decorations.pathreplacing}

\usepackage{float}
\usepackage{enumerate}
\usepackage{array}

\ifpdf
  \DeclareGraphicsExtensions{.eps,.pdf,.png,.jpg}
\else
  \DeclareGraphicsExtensions{.eps}
\fi

\usepackage{enumitem}
\setlist[enumerate]{leftmargin=.5in}
\setlist[itemize]{leftmargin=.5in}

\theoremstyle{definition}
\newtheorem{definition}{Definition}[section]

\newtheorem{theorem}{Theorem}[section]
\newtheorem{proposition}{Proposition}[section]
\newtheorem{lemma}{Lemma}[section]

\theoremstyle{remark}

\usepackage{amsopn}

\newcommand{\R}{\mathbb{R}}
\newcommand{\mc}[1]{\mathcal{#1}}

\DeclareMathOperator{\diam}{diam}
\DeclareMathOperator{\pd}{PD}

\newcommand{\iid}{\stackrel{\text{iid}}{\sim}}

\newcommand{\var}{\mathrm{Var}}

\newcommand{\mathsc}[1]{{\normalfont\textsc{#1}}}
\newcommand{\mcd}[1]{\mathcal{D}_{\text{#1}}}

\title{Uncertainty of Network Topology with Applications to Out-of-Distribution Detection}

\author{
 Sing-Yuan Yeh \\
  Data Science Degree Program \\
  National Taiwan University and Academia Sinica\\
  Taipei 106, Taiwan\\
  \texttt{d10948003@ntu.edu.tw} \\
   \And
 \href{https://orcid.org/0000-0002-2522-5957}{\includegraphics[scale=0.06]{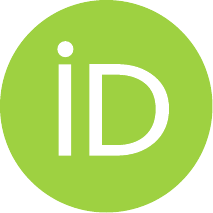}\hspace{1mm} Chun-Hao Yang}\thanks{Corresponding author.} \\
  Institute of Statistics and Data Science\\
  National Taiwan University\\
  Taipei 106, Taiwan\\
  \texttt{chunhaoy@ntu.edu.tw} \\
}

\begin{document}
\maketitle
\begin{abstract}
Persistent homology (PH) is a crucial concept in computational topology, providing a multiscale topological description of a space. It is particularly significant in topological data analysis, which aims to make statistical inference from a topological perspective. In this work, we introduce a new topological summary for Bayesian neural networks, termed the \textit{predictive topological uncertainty (pTU)}. The proposed pTU measures the uncertainty in the interaction between the model and the inputs. It provides insights from the model perspective: if two samples interact with a model in a similar way, then they are considered identically distributed. We also show that the pTU is insensitive to the model architecture. As an application, pTU is used to solve the out-of-distribution (OOD) detection problem, which is critical to ensure model reliability. Failure to detect OOD input can lead to incorrect and unreliable predictions. To address this issue, we propose a significance test for OOD based on the pTU, providing a statistical framework for this issue. The effectiveness of the framework is validated through various experiments, in terms of its statistical power, sensitivity, and robustness.
\end{abstract}

% keywords can be removed
%\keywords{First keyword \and Second keyword \and More}

\section{Introduction}\label{sec:intro}

Over the past few decades, deep neural networks have achieved remarkable successes in various fields, e.g., computer vision, medical image analysis, and natural language processing. However, for models to be reliable, uncertainty must be taken into account. Uncertainty quantification (UQ) is a crucial aspect of machine learning, as it provides insights into the model's reliability and robustness. The UQ methods for deep learning are broadly categorized into three groups: Bayesian methods (e.g., MC dropout \cite{gal2016dropout}, Bayes By Backprop \cite{blundell2015weight}, and many others), ensemble methods (e.g., deep ensemble \cite{lakshminarayanan2017simple,rahaman2021uncertainty}), and others. See \cite{abdar2021review} for a comprehensive review of UQ methods in deep learning. In this work, we focus mainly on Bayesian neural networks.

\paragraph{Bayesian Neural Network} Compared to a typical neural network, a Bayesian neural network requires an additional assumption on the prior distributions on the network weights, which are usually assumed to be Gaussian distributions. Applying the Bayes theorem, one gets the unnormalized posterior distributions. However, due to the ultra-high dimensionality, the intractable normalizing constant, and complicated network structures, direct and exact inference for the posterior distributions of the weights is impossible. An alternative approach is to consider variational inference, which seeks approximation to the actual posterior. For example, the author of \cite{ranganath2014black} proposed the black-box variational inference, which maximizes the evidence lower bound to find the approximation. This is equivalent to minimizing the Kullback-Leibler divergence to the actual posterior. See Table B.3 in \cite{abdar2021review} for a more comprehensive list of methods in this line of research. With approximate posterior distributions for the network weights, the uncertainty is directly associated with the randomness of the posterior distributions, usually defined as the posterior variance or the entropy of the posterior distribution.

\paragraph{Topological Uncertainty} There are two main types of uncertainty: epistemic uncertainty, which arises from the model's lack of knowledge, and aleatoric uncertainty, which arises from the inherent randomness in the data \cite{hullermeier2021aleatoric}. In recent years, there has been a growing interest in applying topological tools, in particular the persistent homology (PH), to analyze neural networks. PH is a crucial concept in computational topology, providing a multiscale topological description of a space. It is particularly significant in topological data analysis, which aims to make statistical inference from a topological perspective. For example, \cite{rieck2019neural} used PH to measure the complexity of neural networks, while \cite{naitzat2020topology} studied the topological properties of neural networks. \cite{lacombe2021topological} introduced the concept of topological uncertainty that describes the uncertainty in the activation graph of a neural network. The notion of activation graph is proposed by \cite{gebhart2019characterizing}. As opposed to static graphs (i.e., the weights in a neural network), activation graphs are able to capture the interaction between the input and the network. However, the uncertainty proposed by \cite{lacombe2021topological} fails to consider the model uncertainty, i.e., the uncertainty in the model parameters. The main goal of this work is to quantify the topological uncertainty of Bayesian neural networks.

\paragraph{Contribution}
In this work, we introduce the concept of \textit{predictive topological uncertainty (pTU)}, which measures the uncertainty in the interaction between the model and the input to be predicted. The pTU is defined as the variation of the persistent homology of the activation graphs. A more precise definition will be given in Section~\ref{sec:pTU}. The pTU provides insights from the perspective of model-input interactions: if two samples interact with a model in a similar way, then they are considered identically distributed. The following statistical inference is made based on the distribution of pTU, rather than the value of pTU. The main reason for this choice is that making inference based on the distribution incorporates the randomness inherent to the inputs. 

As an application, pTU is used to solve the out-of-distribution (OOD) detection problem. OOD is an important issue when applying a trained model to unseen data. Failure to detect OOD inputs can lead to incorrect and unreliable predictions. There are three kinds of OOD: (i) distribution shift (shift in the distributions of both the inputs and the outputs), (ii) covariate shift (shift in the input distribution), and (iii) semantic shift (shift in the conditional distribution of output given input). See \cite{yang2021generalized} for a survey on the research of OOD detection. In this work, we focus on the covariate shift (CS) problem. To handle this, we propose a significance test for CS based on the deviation of the pTU distribution, providing a statistical framework for the detection of CS. 

\textit{To sum up, the main contribution of this work is two-fold: (i) we propose a novel definition of topological uncertainty for Bayesian neural networks, and (ii) we derive a significance test for CS based on pTU and we validate its performance in terms of its statistical power, robustness, and sensitivity.}

The rest of the paper is organized as follows. In Section~\ref{sec:homology}, we briefly review the concept of persistent homology and how to compute the persistent homology of the activation graph of a neural network. In Section~\ref{sec:pTU}, we present the definition of predictive topological uncertainty and its estimator. In Section~\ref{sec:OOD_test}, we propose a significance test for CS based on the pTU. In Section~\ref{sec:experiment}, we evaluate the proposed framework in terms of its sensitivity, robustness, and statistical power and show that our method can effectively detect the presence of CS. Finally, we conclude this paper in Section~\ref{sec:conclusion} by discussing the advantages of our methods and some potential issues that are worth further investigation. 
\section{Persistent Homology of Neural Networks}\label{sec:homology}

In this section, we will first review the concept of persistent homology, as well as its computation and representation. Next, we review the construction of activation graphs \cite{gebhart2019characterizing} from neural networks. 

\subsection{Persistent Homology and Persistence Diagram}

\paragraph{Homology} Homology groups are fundamental constructs in algebraic topology, providing a way to algebraically describe topological spaces. For a topological space $\mc{X}$, the $k$-th homology group $H_k(\mathcal{X})$ of captures information about the $k$-dimensional holes in $\mathcal{X}$. Each of these homology groups provides critical insights into the underlying structure of the topological space, contributing to a comprehensive understanding of its geometric and topological properties. For each $k$, the \emph{$k$-th Betti number}, $\beta_k$, is the rank of the $k$-th homology group. That is, the $k$-th Betti number $\beta_k$ represents the number of $k$-dimensional holes in a topological space. These Betti numbers essentially summarize the topological structure of the space.

\paragraph{Persistent Homology} Given a filtration $\emptyset \subseteq \mc{X}_1 \subseteq \mc{X}_2 \subseteq \cdots \subseteq \mc{X}_n = \mc{X}$, we have the homology groups for each space $H_k(\mc{X}_i)$, $i = 1,\ldots, n$, $k = 0, 1, 2,\ldots$. The inclusion $\mc{X}_i \subseteq \mc{X}_j$ induces homeomorhphisms $f_{k}^{i,j}: H_k(\mc{X}_i) \to H_k(\mc{X}_j)$ for each $k$. The \emph{$k$-th persistent homology group} is the image of $f^{i,j}_k$, denoted by $H^{i,j}_k = \text{im} f_{k}^{i,j}$ for $1 \leq i \leq j \leq n$. The $H_k^{i,j}$ contains $k$-dimensional topological information of $\mc{X}_i$ that are still present in $\mc{X}_j$. The \emph{$k$-th persistent Betti number} $\beta_k^{i, j} = \text{rank}(H_k^{i,j})$ is the number of $k$-dimensional holes in $\mc{X}_i$ that are still in $\mc{X}_j$. Let $\mu_k^{i, j}$ be the number of $k$-dimensional holes born at $\mc{X}_i$ and dying entering $\mc{X}_j$, i.e.,
\begin{align*}
    \mu_k^{i, j}=\left(\beta_k^{i, j-1}-\beta_k^{i, j}\right)-\left(\beta_k^{i-1, j-1}-\beta_k^{i-1, j}\right).
\end{align*}
The collection $\{(i, j, \mu_k^{i,j}): 1 \leq i \leq j \leq n\}$ is called the \emph{$k$-th persistence diagram} and the collection of all $k$-th persistence diagrams, $k = 0, 1, 2, \ldots$, is called the \emph{persistence diagram (PD) of $\mc{X}$}, denoted by $\pd(\mc{X})$. For a more detailed description, see the books by \cite{zomorodian2005topology} and \cite{edelsbrunner2010computational}. 

\paragraph{Computation} One of the commonly used filtrations is the \emph{Vietoris–Rips filtration}. Let $\mathbb{X} = \{x_1, \ldots, x_n\} \subseteq \R^d$ be a point cloud. The Vietoris-Rips complex of $\mathbb{X}$ with diameter $r$ is $\textbf{VR}_r(\mathbb{X}) = \{ S \subseteq \mathbb{X}: \diam S \leq r\}$, where $\diam S = \max_{x_i, x_j \in S} \|x_i - x_j\|$. With $0 \leq r_1 \leq r_2 \leq \cdots \leq r_m$, we have the Vietoris-Rips filtration
$$
    \emptyset \subseteq \textbf{VR}_{r_1}(\mathbb{X}) \subseteq \textbf{VR}_{r_2}(\mathbb{X}) \subseteq \cdots \subseteq \textbf{VR}_{r_m}(\mathbb{X}).
$$
The $k$-th PD of $\mathbb{X}$, denoted by $\pd_k(\mathbb{X})$, is the collection of $(r_i, r_j, \mu_k^{i,j})$; that is, the number of $k$-dimensional holes that are born at $r_i$ and die at $r_j$ is $\mu_k^{i,j}$. There are other ways to represent the persistent homology, for example persistence images \cite{adams2017persistence} and persistence landscapes \cite{bubenik2015statistical}. By viewing PDs as empirical distributions, the commonly used distance metrics, the Wasserstein distance \cite{cohen2010lipschitz} and the bottleneck distance \cite{cohen2007stability}, can be used to measure the distance between persistence diagrams.

\subsection{Activation Graph}\label{sec:pd}

Let $F_{\theta}$ be a $L$-layer neural network; that is, for $x \in \R^{h_0}$,
\[
    F_{\theta}(x) = f_{L}\circ f_{L-1} \circ \cdots \circ f_{2} \circ f_1(x)
\]
where $f_{\ell}: \R^{h_{\ell-1}} \to \R^{h_{\ell}}$ is defined by $f_{\ell}(x) = \sigma(W_{\ell}x + b_{\ell})$ and $\sigma: \R \to \R$ is called an activation function, which is an element-wise operation on $\R^{h_{\ell}}$. Here, the number of neurons in the $\ell$-th layer is denoted $h_{\ell}$, ${\ell} = 0,\ldots, L$ with $h_0$ and $h_{L}$ being the dimension of input and output. The parameter to be learned is 
\[
    \theta = \left\{(W_{\ell}, b_{\ell}): W_{\ell} \in \R^{h_{\ell}\times h_{\ell-1}}, b_{\ell} \in \R^{h_{\ell}}\right\}_{\ell=1}^L.
\] 
Every two consecutive layers can be treated as a \textit{weighted bipartite graph}: the nodes are the neurons in the $(\ell-1)$-th and $\ell$-th layers and the edge weights are given by the absolute value of the model weights $W_{\ell}$. This is referred to as the static graph since it depends only on the model weights. The author of \cite{gebhart2019characterizing} introduced the concept of \textit{activation graph} in which the edge weights are defined based on an input and the model weights. Specifically, if $x_{\ell-1} \in \R^{h_{\ell-1}}$ is an input to the $\ell$-th layer of wights, then the weight on the edge connecting the $i$-th neuron in $(\ell-1)$-th layer and the $j$-th neuron in the $\ell$-th layer is $|W_{\ell}(j,i)x_{\ell-1}(i)|$ where $x_{\ell-1}(i)$ is the $i$-th entry of $x_{\ell-1}$ and $W_{\ell}(j,i)$ is the $(j,i)$-th entry of $W_{\ell}$. The quantities $|W_{\ell}(j,i)x_{\ell-1}(i)|$ characterize the interaction between the input and the model and the network topology. Formally, let the vertex set $V_\ell$ collect the neurons from both the $(\ell-1)$-th and $\ell$-th neurons and the edge set $E_\ell$ collect all connections between the $(\ell-1)$-th and $\ell$-th neurons. Moreover, each edge $(i,j)\in E_\ell$ endows a weight $|W_{\ell}(j,i)x_{\ell-1}(i)|$. The weighted bipartile graph $(V_\ell,E_\ell,\phi_\ell)$ is the activation graph, where $\phi_\ell:E\to \mathbb{R}_{\geq 0}$ defined by $\phi_\ell(i,j)=|W_{\ell}(j,i)x_{\ell-1}(i)|$. See Figure~\ref{fig:network} for an illustration.

To compute the PH of the activation graphs, we follow the construction proposed by \cite{rieck2019neural}. First, a filtration is defined based on the magnitude of the edge weights. Given this filtration, the zero-dimensional PD is equivalent to the empirical distribution of the weights of its maximum spanning tree\footnote{Note that we only collect the zero-dimensional topological feature since it is impossible for bipartite graphs to contain a circle. Also, all the zero-dimensional topological features have the same birth, and hence the zero-dimensional PD is equivalent to the empirical distribution of the weights.} (MST). We denote by $D_{\ell}(x_{\ell-1}, F_{\theta})$ the PD of the $\ell$-th activation graph of a neural network $F_{\theta}$ and an input $x_{\ell-1}$ (to the $\ell$-th layer). For more details about this procedure, please refer to \cite{lacombe2021topological, rieck2019neural}.

\begin{figure}[ht]
\centering
\scalebox{0.8}{ % 缩小到0.8倍
\begin{tikzpicture}[
    neuron/.style={circle, draw=black!70, thick, minimum size=8mm}, % 改為黑色節點樣式
    connect/.style={draw=gray!70, thin}, % 灰色連線樣式
    brace/.style={decorate, decoration={brace, amplitude=10pt}, thick, black, align=center}, % 大括號樣式
    edge below/.style={preaction={draw=white, -, line width=3pt}} % 將邊線移到底層
]

% 設定層間距
\def\layersep{2.5cm}

% 第一層節點 (4 nodes)
\foreach \i in {1,...,3} {
    \node[neuron] (A\i) at (0, -\i-1) {}; % 水平置中
}

% 第二層節點 (6 nodes)
\foreach \i in {1,...,5} {
    \node[neuron] (B\i) at (\layersep, -\i) {};
}

% 第三層節點 (6 nodes)
\foreach \i in {1,...,5} {
    \node[neuron] (C\i) at (2*\layersep, -\i) {};
}

% 第四層節點 (6 nodes)
\foreach \i in {1,...,5} {
    \node[neuron] (D\i) at (3*\layersep, -\i) {};
}

% 第五層節點 (4 nodes)
\foreach \i in {1,...,3} {
    \node[neuron] (E\i) at (4*\layersep, -\i-1) {}; % 水平置中
}

% 連接第一層和第二層
\foreach \i in {1,...,3} {
    \foreach \j in {1,...,5} {
        \draw[connect] (A\i) -- (B\j); % 將連線放置到底層
    }
}

% 連接第二層和第三層
\foreach \i in {1,...,5} {
    \foreach \j in {1,...,5} {
        \draw[connect] (B\i) -- (C\j); % 將連線放置到底層
    }
}

% 連接第三層和第四層
\foreach \i in {1,...,5} {
    \foreach \j in {1,...,5} {
        \draw[connect] (C\i) -- (D\j); % 將連線放置到底層
    }
}

% 連接第四層和第五層
\foreach \i in {1,...,5} {
    \foreach \j in {1,...,3} {
        \draw[connect] (D\i) -- (E\j); % 將連線放置到底層
    }
}

% 添加層注釋 (x-th layer of neurons)
\node[below=1.18cm of A3.south] {0-th layer $x_0$}; % 第一層
\node[below=0.2cm of B5.south] {1-st layer $x_1$}; % 第二層
\node[below=1.18cm of E3.south] {$L$-th layer $x_L$}; % 第三層

% 添加大括號 (第一層和第二層)
\draw[brace] ([yshift=1.1cm] A1.north) -- ([yshift=0.1cm] B1.north)
    node[midway, above=5pt, black] {1-st layer of \\ weights / bipartite graph};

% 添加大括號 (第二層和第三層)
\draw[brace] ([yshift=0.1cm] D1.north) -- ([yshift=1.1cm] E1.north)
    node[midway, above=5pt, black] {$L$-th layer of \\ weights / bipartite graph};

\node at ([xshift=1.3cm, yshift=-0.5cm] A2) {{\Large $f_1(x_0)$}};
\node at ([xshift=1.3cm, yshift=-0.5cm] B3) {{\Large $f_2(x_1)$}};
\node at ([xshift=1.3cm, yshift=-0.5cm] D3) {{\Large $f_L(x_{L-1})$}};
\end{tikzpicture}}
\caption{$L$-layer neural network diagram with $L=4$.}\label{fig:network}
\end{figure}
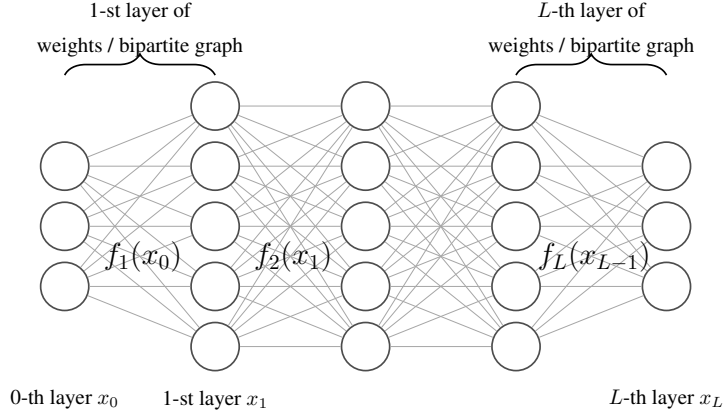

% \begin{algorithm}
%   \caption{Persistence Diagram for the $\ell$-th Activation Graph \cite{lacombe2021topological, rieck2019neural}}
%   \label{alg:PD_AG}
%   \begin{algorithmic}[1]
%         \Function{PDAG}{$\ell$, $x$, $F_{\theta}$}
%         \State $D_{\ell}(x, F_{\theta})$ is the vector of weights of $\text{MST}(A_{\ell})$ in descending order 
%         \State Return $D_{\ell}(x, F_{\theta})$
%     \EndFunction
%   \end{algorithmic} 
% \end{algorithm}
\section{Topological Uncertainty}\label{sec:pTU}

In this section, we present our definitions of topological uncertainty for a Bayesian neural network. We define two notions of topological uncertainty: the \textit{predictive Topological Uncertainty (pTU)} (Definition~\ref{def:pTU}) and the \textit{model Topological Uncertainty (TU)} (Definition~\ref{def:TU}). The main difference between these two notions is that the pTU is defined with respect to a particular input (sample) and a model, whereas the TU is defined with respect to a model. 

\subsection{Predictive Topological Uncertainty}

Suppose $F_{\theta}$ is a neural network trained using the dataset $\mc{D}_{\text{train}}$. We denote the (approximated) posterior distribution of the parameters by $\pi(\theta \mid \mc{D}_{\text{train}})$. Given $x_0 \in \R^{h_0}$ and a parameter $\theta$ drawn from $\pi(\theta \mid \mc{D}_{\text{train}})$, we construct the PDs $D_{\ell}(x_{\ell-1}, F_{\theta})$ where $x_{\ell-1} = f_{\ell-1} \circ \cdots \circ f_{1}(x)$ and $\ell = 1,\ldots, L$. Note the PDs $D_{\ell}(x_{\ell-1}, F_{\theta})$ have inherent randomness from both the posterior distribution of $\theta$ and the sampling distribution of $x$. The definition of pTU is given below.

\begin{definition}[Predictive Topological Uncertainty]\label{def:pTU}
The predictive topological uncertainty (pTU) of an input $x_0 \in \R^{h_0}$ and a (trained) Bayesian neural network $F_{\theta}$ is defined as 
\begin{equation}\label{eq:ptu}
\mathbf{pTU}(x \mid \mc{D}_{\text{train}}) = \frac{1}{L}\sum_{\ell=1}^L\var_{\theta}\left[D_{\ell}(x_{\ell-1}, F_{\theta}) \mid \mc{D}_{\text{train}}\right]
\end{equation}
where $x_{\ell-1} = f_{\ell-1} \circ \cdots \circ f_{1}(x)$ is the input of the $(\ell-1)$-th layer.
\end{definition}

The $\var(\cdot)$ in the Definition~\ref{def:pTU} is the Fr\'{e}chet variance \cite{frechet1948elements} taken with respect to the posterior distribution $\pi(\theta \mid \mc{D}_{\text{train}})$, i.e.,
\begin{equation}\label{eq:var_Dl}
\var_{\theta}\left[D_{\ell}(x_{\ell-1}, F_{\theta}) \mid \mc{D}_{\text{train}}\right] = \inf_{m} \int \text{Dist}(D_{\ell}(x_{\ell-1}, F_{\theta}), m)^2 \pi(\theta \mid \mc{D}_{\text{train}}) d\theta
\end{equation}
where $\text{Dist}(\cdot, \cdot)$ is a distance metric for PDs and the infimum is taken over all possible PDs. In practice, the posterior distribution is often intractable and hence the posterior Fr\'{e}chet variance can only be approximated by
\begin{equation}\label{eq:sample_var_Dl}
\widehat{\var_{\theta}}\left[D_{\ell}(x_{\ell-1}, F_{\theta}) \mid \mc{D}_{\text{train}}\right] = \inf_{m} \sum_{i=1}^k \text{Dist}(D_{\ell}(x_{\ell-1}, F_{\theta_i}), m)^2 
\end{equation}
where $\theta_1, \ldots, \theta_k \iid \pi(\theta \mid \mc{D}_{\text{train}})$. In this work, we choose the 2-Wasserstein distance for the PDs. The main reason for this choice is that, under the 2-Wasserstein distance, Eq.~\eqref{eq:sample_var_Dl} can be efficiently computed\footnote{Since the zero-dim PDs can be viewed as one-dimensional empirical distributions, the Wasserstein distance has a closed form expression and so do the Wasserstein barycenter and variance.} as described in Algorithm~\ref{alg:pTU}. 

\begin{algorithm}[ht]
  \caption{Predictive Topological Uncertainty}
  \label{alg:pTU}
  \begin{algorithmic}[1]
    \Function{pTU}{$x$, $F_{\theta}$}
    \State Sample $\theta_1, \ldots, \theta_m \iid \pi(\theta \mid \mc{D}_{\text{train}})$
    \State $D_{\ell}^{(i)} \coloneq D_{\ell}(x_{\ell-1}, F_{\theta_i}) \quad \text{for all }i=1,\ldots m, \ell=1,\ldots,L$ \Comment{Compute PDs by Section \ref{sec:pd}}
    \For{$\ell = 1, \ldots, L$}
    \State Compute $\bar{D}_{\ell} = \frac{1}{m}\sum_{i=1}^m D_{\ell}^{(i)}$ \Comment{The Fr\'{e}chet mean w.r.t $W_2$}
    \State Compute $\hat{\sigma}^2_{\ell} \coloneqq \frac{1}{m}\sum_{i=1}^m \|D_{\ell}^{(i)} - \bar{D}_{\ell}\|^2$
    \Comment{The Fr\'{e}chet variance w.r.t $W_2$}
    \EndFor
    \State Return $\frac{1}{L}\sum_{\ell=1}^{L} \hat{\sigma}^2_{\ell}$
    \EndFunction
  \end{algorithmic} 
\end{algorithm}

The intuition behind this definition of pTU is that for any Bayesian model, the uncertainty is captured by the posterior variance of the parameters, and the predictive uncertainty is captured by the variance of the posterior predictive distribution. Hence, we use the posterior variance of the PD as a measure for the topological uncertainty.

\subsection{Model Topological Uncertainty}

The pTU is defined based on an input $x$ and a Bayesian neural network $F_{\theta}$ and $\pi(\theta \mid \mc{D}_{\text{train}})$. If we consider $X \sim P$, then $\mathbf{pTU}(X \mid \mc{D}_{\text{train}})$ is a positive-valued random variable. The distribution of pTU is determined by both the posterior distribution of $\theta$ and the sampling distribution $P$. Hence, we can define the model topological uncertainty by taking the expectation of pTU with respect to the sampling distribution $P$.

\begin{definition}[Model Topological Uncertainty]\label{def:TU}
The model topological uncertainty is defined as 
\begin{equation}\label{eq:TU}
\mathbf{TU}(P \mid \mc{D}_{\text{train}}) = \mathbb{E}_P\left[\mathbf{pTU}(X \mid \mc{D}_{\text{train}})\right]
\end{equation}
where the expectation is taken with respect to $X \sim P$.
\end{definition}

The $\mathbf{TU}(P \mid \mc{D}_{\text{train}})$ can be estimated by the empirical mean
\begin{equation}\label{eq:empirical_TU}
    \widehat{\mathbf{TU}}(P \mid \mc{D}_{\text{train}}) = \frac{1}{n}\sum_{i=1}^n \mathbf{pTU}(X_i \mid \mc{D}_{\text{train}})
\end{equation}
where $X_1, \ldots, X_n \iid P$. In the context of machine learning, the distribution $P$ is often referred to as the testing distribution. Therefore, the empirical TU Eq.~\eqref{eq:empirical_TU} is a measure of topological uncertainty of the testing distribution with respect to the model $F_{\theta}$ and the training data $\mc{D}_{\text{train}}$. There are a few remarks regarding pTU and TU: 
\begin{enumerate}
    \item Even if the test data $X_1, \ldots, X_n \iid P$ is identical to the training data, the empirical TU Eq.~\eqref{eq:empirical_TU} would not be zero. This makes sense since the model has limited capacity.
    \item In the context of Bayesian inference, the training data $\mc{D}_{\text{train}}$ are treated as given, and therefore the distribution of training data is not relevant. 
    \item Suppose now we have two distributions $P_1$ and $P_2$. The similarity of the pTU's distributions indicates that the network topology activated by $P_1$ and $P_2$.
    \item However, if the distributions of pTUs are diverging, it suggests that the two distributions $P_1$ and $P_2$ are different. A formal mathematical result is provided in the subsequent section.
\end{enumerate}
Based on these observations, we propose to use pTU and TU for out-of-distribution (OOD) detection, see Section~\ref{sec:OOD_test}.

% {\color{red}$\|\widehat{pTU}(\{x_i\}|\mathcal{D}_\text{train})-\mathbf{pTU}(\mathcal{X}|\mathcal{D}_\text{train})\|_\infty\to 0$? a.s.}

\subsection{Stability of pTU under the Wasserstein distance}

The computation of pTU depends heavily on the distance metrics for PD, for example, the $p$-Wasserstein distance or the bottleneck distance. 
These metrics are motivated by stability theorems. In topological data analysis, the stability theorems \cite{Chazal2016,skraba2020wasserstein, lacombe2021topological} are essential for preserving the robustness of persistent diagrams, which allow us to reliably capture the underlying structures in data. In our work, the $\mathbf{pTU}(X \mid \mc{D}_{\text{train}})$ can be viewed as a function of the random variable $X \sim P$. As our main application for pTU is the OOD detection, the stability, or Lipschitzness, of the mapping $x \mapsto \mathbf{pTU}(x \mid \mc{D}_{\text{train}})$ is important, serving as the theoretical basis for the significance test for OOD in Section~\ref{sec:OOD_test}. Therefore, we provide Theorem~\ref{thm:stab_thm} to show the stability property of pTU.

Before stating the theorem, we recall that, for two distributions $\mu$ and $\nu$, the $p$-Wasserstein distance is defined by
\begin{equation*}
W_p(\mu, \nu)=\inf _{\gamma \in \Gamma(\mu, \nu)}\left(\mathbb{E}_{(x, y) \sim \gamma} c(x, y)^p\right)^{1 / p}
\end{equation*}
where $\Gamma(\mu, \nu)$ is the set of all couplings of $\mu$ and $\nu$, and $c(\cdot, \cdot)$ is the cost function between $x$ and $y$. The Wasserstein distance is motivated by the optimal transport problem, and the cost function $c(\cdot, \cdot)$ can be chosen according to the specific application. One common choice is an appropriate distance metric defined on the support of $\gamma$. In the following analysis, we use the $L_{\infty}$ norm as the cost function, i.e., $c(x, y) = \|x - y\|_{\infty}$.

We make a few assumptions for the Bayesian neural network $F_{\theta}$ and the training data $\mc{D}_{\text{train}}$:
\begin{enumerate}[label=(\textbf{A.\arabic*})]
    \item The activation function $\sigma$ is 1-Lipschitz.
    \item The posterior distribution $\pi(\theta \mid \mc{D}_{\text{train}})$ is asymptotically normal.
    \item The training dataset $\mc{D}_{\text{train}}$ is in a bounded domain.
\end{enumerate}

These assumptions are easy to satisfy in practice. For example, the commonly used activation functions $\text{ReLU}(x) = \max(x, 0)$ and sigmoid $\sigma(x) = (1+\exp(-x))^{-1}$ are 1-Lipschitz. The posterior normality assumption follows from the Bernstein-von Mises Theorem, see \cite{liu2021variable} for example.

\begin{theorem}[Stability of pTU]\label{thm:stab_thm}
Suppose that there are random variables $X \sim P_X$ and $Y \sim P_Y$. Under the assumptions (\textbf{A.1}) -- (\textbf{A.3}), the distribution of $\mathbf{pTU}$ is stable with respect to the $p$-Wasserstein distance for $p \geq 1$, i.e.,
\begin{equation*}
W_p(\mathbf{pTU}(X \mid \mc{D}_{\text{train}}), \mathbf{pTU}(Y \mid \mc{D}_{\text{train}})) \leq C W_p(X, Y).
\end{equation*}
\end{theorem}

Note that the Lipschitz constant $C$ depends on the network architecture, weights, and the radius of the input domain. The proof of Theorem~\ref{thm:stab_thm} is provided in Section~\ref{sec:proof}. We shall mention that the Lipchitz constant might not be optimal. However, these assumptions are typically easy to satisfy in practice.

% \subsection{Estimating the distribution of pTU}
% As mention in above subsection, pTU is 1-dim random variable. However, 
% Note the the proposed pTU is a random variable and we are interested in the conditional distribution of $\mathbf{pTU}(X \mid \mc{D}_{\text{train}})$ given the training dataset.
% \begin{itemize}
%     \item empirical cdf
%     \item kernel density estimation
% \end{itemize}
\section{Significance Test for Out-of-Distribution}\label{sec:OOD_test}

In this section, we propose a significance test for OOD detection based on pTU. Given a model $F_{\theta}$ trained on the dataset $\mc{D}_{\text{train}}$, we want to perform a statistical test for two random variables $X \sim P_1$ and $X^{\prime} \sim P_2$ to test 
\begin{equation}\label{eq:hypothesis}
H_0: P_1 = P_2 \quad \text{v.s.} \quad H_1: P_1 \neq P_2.
\end{equation}
This is a two-sample problem, and numerous research works on this problem, for example, the kernel two-sample test based on maximum mean discrepancy \cite{gretton2012kernel} and Stein discrepancy \cite{liu2016kernelized}. However, the standard two-sample problem does not take into account the interaction between the model and the sample. Therefore, we propose to use 
\[
W_p(\mathbf{pTU}(X \mid \mc{D}_{\text{train}}), \mathbf{pTU}(X^{\prime} \mid \mc{D}_{\text{train}}))
\]
as the test statistic for testing Eq.~\eqref{eq:hypothesis}. 

In practice, we approximate the pTU distribution using the dataset $\{X_i\}_{i=1}^m$. From Algorithm~\ref{alg:pTU}, we obtain the empirical distribution $\{\widehat{\textbf{pTU}}(X_i \mid \mathcal{D}_{\text{train}})\}_{i=1}^n$, which allows us to directly compute the discrete Wasserstein distance between two distributions. That is, if $P$ and $Q$ are two empirical distributions of samples $Y_1,\ldots, Y_n$ and $Z_1,\ldots, Z_n$ respectively, then the discrete $p$-Wasserstein distance is computed based on the order statistics:
\begin{equation*}
W_p(P, Q):=\left(\sum_{i=1}^n\left\|Y_{(i)}-Z_{(i)}\right\|^p\right)^{1 / p}\,.
\end{equation*}

If we collect two datasets $X_1, \ldots, X_m \iid P_1$ and $X^{\prime}_1, \ldots, X^{\prime}_m \iid P_2$, we can compute
$\{\widehat{\mathbf{pTU}}(X_i \mid \mcd{train})\}_{i=1}^m$ and $\{\widehat{\mathbf{pTU}}(X^{\prime}_i \mid \mcd{train})\}_{i=1}^m$. Then we use a permutation test with the test statistic being the discrete $p$-Wasserstein distance to test \eqref{eq:hypothesis}. The procedure is described in Algorithm~\ref{alg:pTU_test}. If the resulting $p$-value is less than $\alpha = 0.05$, we reject the null hypothesis ($H_0$), concluding that $P_1 \neq P_2$.

\begin{algorithm}
    \caption{Two-sample pTU Test}
    \label{alg:pTU_test}
    \begin{algorithmic}[1]
    \Require $X_1, \ldots, X_m \iid P_1$, $X^{\prime}_1, \ldots, X^{\prime}_m \iid P_2$, $M \in \mathbb{N}$
    \State $\mathtt{X}=\mathbf{pTU}(X_i \mid \mc{D}_{\text{train}})$ and $\mathtt{Y}=\mathbf{pTU}(X^{\prime}_i \mid \mc{D}_{\text{train}})$
    \State $\mathtt{combined\_data} = \mathsc{concatenate}(\mathtt{X}, \mathtt{Y})$
    \State $T_{\text{obs}} = W_p(\mathtt{X}, \mathtt{Y})$ 
        \For{$i = 1, \ldots, M$}
        % \State \texttt{shuffled_data} = \textsc{shuffle}(\texttt{combined_data})
        \State $\mathtt{new\_X} = \mathtt{shuffled\_data}[:\mathsc{len}(\mathtt{X})]$
        \State $\mathtt{new\_Y} = \mathtt{shuffled\_data}[\mathsc{len}(\mathtt{X}):]$
        \State $T^{*}_i = W_p(\mathtt{new\_X}, \mathtt{new\_Y})$
        \EndFor
    \State $p\mathrm{-value}=\frac{1}{M}\sum_{i=1}^M \mathbb{I}(T^{*}_i \geq T_{\text{obs}})$
    \end{algorithmic} 
\end{algorithm}

\section{Experiments}\label{sec:experiment}
In this section, we demonstrate the application of pTU on image classification tasks. 
\paragraph{Datasets and preprocessing.} We utilize publicly available image datasets with 10 classes, namely \texttt{SVHN}, \texttt{MNIST}, \texttt{Fashion-MNIST}, and \texttt{CIFAR-10}. Please see Appendix for more details. All images are resized to $28\times 28\times 1$, converted to grayscale, and rescaled from a range of $[0, 255]$ to $[0, 1]$. When we feed the images into our model, we reshape them to $784 \times 1$.

\paragraph{Bayesian models setting.} In these experiments, we use a Bayesian model based on the work of \cite{blundell2015weight}, which is known as \textit{Bayes-by-Backprop} algorithm. We leverage \textit{Bayes-by-Backprop} to train the model. For each layer, the prior distribution is set as a mixture of Gaussians, described in \cite{blundell2015weight}, with $\sigma_1=1$ and $\sigma_2=0.007$. The network architecture is implemented as multi-layer perceptron (MLP) with an input layer of 784 neurons, two hidden layers of rectified linear neurons, and a softmax output layer with 10 neurons. For each input, we use Monte Carlo sampling to draw two parameters from posterior $\pi(\theta|\mathcal{D}_{\text{train}})$ and ensemble these two models. That is, we combine the predictions of both models to make a final decision, by averaging their outputs.

% \paragraph{Empirical pTU}
% We approximate exact variance (\ref{eq:var_Dl}), by Monte Carlo sample drawing 30 $\theta_i$'s and computing the empirical variance
% \begin{equation*}
% \var[D_\ell(x,F_\theta)|x,\mathcal{D}_{\mathrm{train}}]=\inf_m\frac{1}{N}\sum^{N}_{i=1}\dist (D_\ell(x,F_{\theta_i}),m)\,.
% \end{equation*}
% where $N=30$ and $D_\ell(x,F_{\theta_i})$ is a PD constructed by $\ell$-th layer activated graph. Then, according to Equation (\ref{eq:ptu}), we calculate the pTU by averaging the variances across all layers. In each dataset, we randomly select 100 data points $\{x_i\}_{i=1}^{100}$ as the test set and we have $\{\mathbf{pTU}(x_i)\}_{i=1}^{100}$. In this section, we will focus on the empirical distribution of $\{\mathbf{pTU}(x_i)\}_{i=1}^{100}$.

\subsection{Sensitivity to Network Architecture}\label{sec:arch}
The following experiment shows the behavior of pTU from models with different number of hidden layers. To show the relationship between pTU and the model architecture, we train models with different numbers of layers on the \texttt{Fashion MNIST} dataset. We have a total of 7 models, with the number of layers ranging from 2 to 8, and the total number of neurons being 874. These models are denoted as M3, M4, ..., M9. We use Algorithm~\ref{alg:pTU} and set the number of samples to 30. In this subsection, we focus not only on pTU but also the variance $\var(D_\ell|x_i,\mathcal{D}_{\mathtt{F-MNIST}})$ in each layer, which corresponds to $\hat{\sigma}^2_\ell$ in the algorithm for $\ell=1,\cdots,L$. The distributions of these statistics are shown in Figure \ref{fig:arch_layer_ptu}.

We can see that for each model, the distribution of empirical variances, $\hat{\sigma}_L^2$, in the last layer is wider than in the other layers. As the layer depth increases, more nonlinear functions are composited, which allows the model to capture more complex or abstract features. This leads to a wider distribution of $\mathbf{pTU}(X|\mathcal{D}_{\texttt{F-MNIST}})$. This indicates that the model is more flexible and capable of carrying more information. 

\begin{figure}[!ht]
\begin{center}
\includegraphics[width=0.95\textwidth]{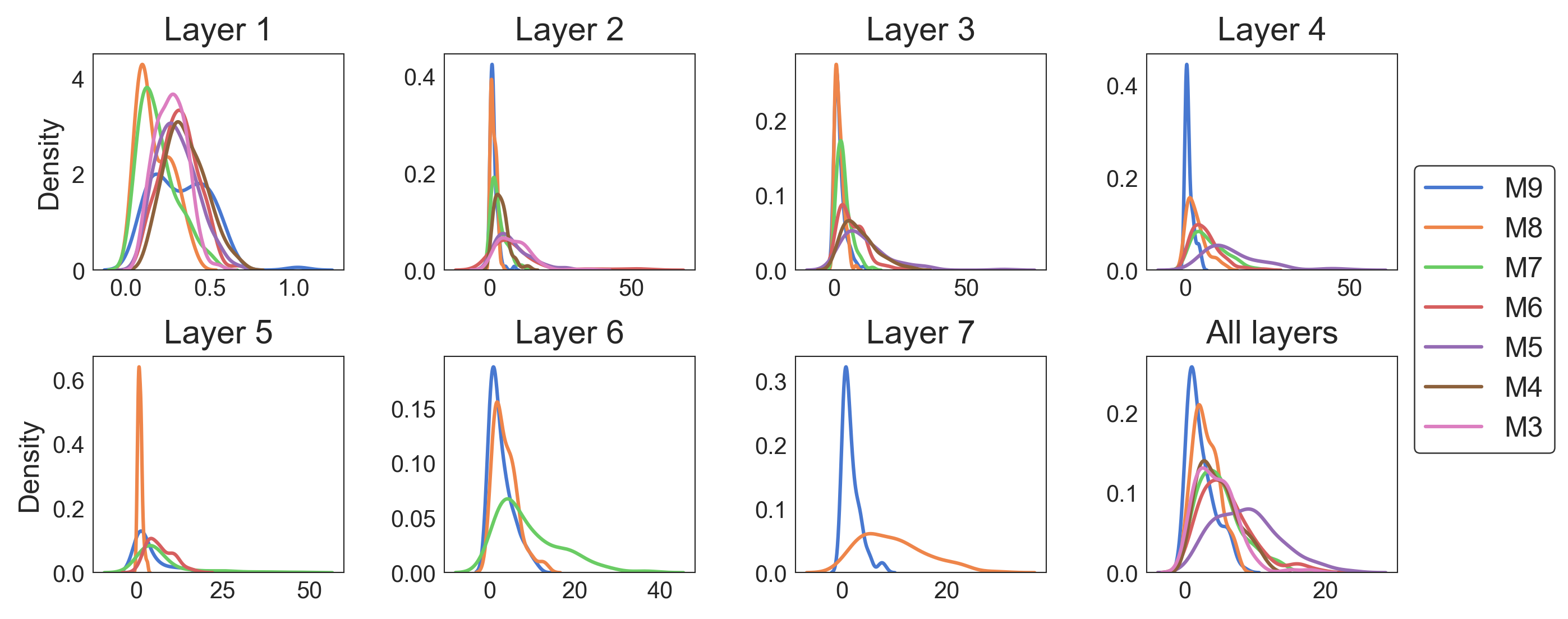}
\end{center}
\caption{From left to right, top to bottom, the seven images: the distribution of empirical variance $\{\hat{\sigma}^2_\ell\}_i$ for $\ell=1,\cdots,7$. Each color of line stands for one model. The eighth image: the distribution of $\{\widehat{pTU}(x_i|\mathcal{D}_{\mathtt{F-MNIST}})\}_i$. }
\label{fig:arch_layer_ptu}
\end{figure}

% \begin{figure}[H]
% \begin{center}
% \includegraphics[width=0.85\textwidth]{Figs/arch_layer.png}
% \end{center}
% \caption{From left to right: the distribution of pTU produced by the $i$-th and $(i+1)$-th where $i=1,\cdots,6$, respectively. ``M$j$L$i$'' stands for the model with $j$ layers and we focus on the pTU between the $i$-th and $(i+1)$-the layer.}
% \label{fig:arch_layer_ptu}
% \end{figure}

% \begin{figure}[H]
% \begin{center}
% \includegraphics[width=0.5\textwidth]{Figs/arch_ptu.png}
% \end{center}
% \caption{The distribution of mean of pTU produced by all layers. ``M$j$'' stands for the model with $j$ layers.}
% \label{fig:}
% \end{figure}

\subsection{Detection of out-of-distribution samples}\label{sec:ood_exp}
The following experiment demonstrates how the distribution of pTU can be used to determine if a trained network faces an out-of-distribution (OOD) dataset. To illustrate this approach, we train three models on \texttt{SVHN}, \texttt{MNIST}, and \texttt{Fashion-MNIST} individually, and then use \texttt{SVHN}, \texttt{MNIST}, \texttt{Fashion-MNIST}, and \texttt{CIFAR-10} to assess whether they are OOD. Let $P(X)$ be the underlying input distribution of training dataset mentioned above and $P(X')$ be the underlying input distribution of new dataset mentioned above. These two distributions may or may not be identical. Let $\widehat{pTU}(x_i|\mathcal{D}_{\text{train}})$, $\widehat{pTU}(x'_i|\mathcal{D}_{\text{train}})$ be the realization of $\mathbf{pTU}(X|\mathcal{D}_{\text{train}})$, $\mathbf{pTU}(X'|\mathcal{D}_{\text{train}})$ where $X\sim P(X)$ and $X'\sim P(X')$, obtained by Algorithm~\ref{alg:pTU} with $m=30$. Given two sets $\{\widehat{pTU}(x_i|\mathcal{D}_{\text{train}})\}_{x_i\in\mathcal{X}}$ and $\{\widehat{pTU}(x'_i|\mathcal{D}_{\text{train}})\}_{x_i\in\mathcal{X'}}$, Algorithm~\ref{alg:pTU_test} is applied to test whether new dataset $\mathcal{X'}$ is OOD. This allows us to use the $p$-value to distinguish two empirical distributions of pTU following the same distribution or not. We show the results in Table~\ref{tab:ood_pval}. If a significance level of 0.05 is chosen, we can observe that the OOD datasets are correctly identified except \texttt{CIFAR} with respect to \texttt{SVHN}. According to the test accuracy in Table~\ref{tab:acc_set}, we can see that the model's performance on SVHN is not ideal, with an accuracy of only 70.36\%. It can be inferred that it did not ``activate'' enough neurons, resulting in an insignificant difference in pTU distributions. 
% However, if the dataset is in-distribution, the $p$-value will be smaller than 0.05.

\begin{table}[h]
\caption{The $p$-values from the permutation test comparing the two distributions are listed in this table. The significance levels are indicated by stars: * $p < 0.05$, ** $p < 0.01$, *** $p < 0.001$.}
\begin{center}
\begin{tabular}
{
>{\centering\arraybackslash}m{2.7cm} ||
>{\centering\arraybackslash}m{2.2cm}
>{\centering\arraybackslash}m{2.2cm}
>{\centering\arraybackslash}m{2.2cm}
>{\centering\arraybackslash}m{2.2cm}
>{\centering\arraybackslash}m{2.2cm}
}
\noalign{\hrule height 1.5pt}
$p$-value  & \multicolumn{4}{c}{Test data $\mathcal{X}'$ }\\\cline{1-1}
Training data $\mathcal{X}$  & \texttt{SVHN} & \texttt{MNIST} & \texttt{F-MNIST} & \texttt{CIFAR}  \\
\noalign{\hrule height 1.5pt}
\texttt{SVHN} & 0.976   &0.001***    &0.012**  &0.026*\\
\texttt{MNIST} & 0.001***  &0.946    &0.001***  &0.001***\\
\texttt{F-MNIST} & 0.001*  &0.001***    &0.053  &0.001***\\
\noalign{\hrule height 1.5pt}
\end{tabular}
\end{center}
\label{tab:ood_pval}
\end{table}

% \begin{figure}[H]
% \begin{center}
% \includegraphics[width=0.95\textwidth]{Figs/ood_conf.png}
% \end{center}
% \caption{From left to right: four models are trained using datasets \texttt{svhn}, \texttt{cifar10}, \texttt{mnist}, and \texttt{Fashion mnist}. Then, these four datasets are inputted to observe the differences in pTU distributions. The numerical values in the tables represent the p-values obtained from the permutation tests, used to assess the differences in these distributions. Therefore, for the first subfigure, the dataset \texttt{svhn} is considered in-domain data, while the other three datasets are considered out-domain data. The similar concepts are applied to the remaining three subfigures.}
% \label{fig:}
% \end{figure}

%%

\subsection{Sensitivity to shifts in sample distribution}\label{sec:shift_exp}
In this experiment, we aim to investigate the impact on the distribution of $\mathbf{pTU}(X|\mathcal{D}_{\texttt{F-MNIST}})$, where $X \sim P(X)$, when the dataset undergoes perturbations from the original dataset. In particular, we consider a family of shifts $\{s_\mu\}_\mu$ defined by $s_\mu(x)=x+\mathcal{N}(\mu\mathbf{1}_{784},1/9I_{784\times 784})$ where $x\in\mathbb
R^{784}$ and $\mathbf{1}_{784}$ is a vector of size 784 with all elements equal to 1. The perturbed dataset is obtained by first drawing $\mathcal{X}=\{x_i\}_{i=1}^n$ from the original dataset \texttt{Fashion-MNIST}, and then use the shift operator to gradually move the distribution away from $\mathcal{X}$, i.e. $\{s_\mu(x_i)\}_{i=1}^n$ for $\mu\in [0, 0.5]$. Then, by Algorithm~\ref{alg:pTU} with $m=30$, we can obtain $\{\widehat{pTU}(s_\mu(x_i)|\mathcal{D}_{\mathtt{F-MNIST}})\}_{x_i\in\mathcal{X}}$ and we can compare the difference of the empirical distribution of $\widehat{pTU}(s_0(X)|\mathcal{D}_{\mathtt{F-MNIST}})$ and $\widehat{pTU}(s_\mu(X)|\mathcal{D}_{\mathtt{F-MNIST}})$ where $\mu\in(0,0.5]$ by permutation test. If the $p$-value is smaller than 0.05, then we will reject $H_0$, and accept $\{s_\mu(x_i)\}$ is OOD. The process is repeated 100 times and the result is shown in the blue curve in Figure~\ref{fig:shift_power}, with each cyan point representing the proportion of rejecting $H_0$ under the corresponding shift $\mu$. Besides, the standard deviation is given by $\sqrt{\hat{p}(1-\hat{p})/n}$ where $n=100$. The curves represent the corresponding moving averages of each experiment. As expected, when the shift $\mu$ increases, the probability of rejecting $H_0$ also increases.

\subsection{Robustness Evaluation}\label{sec:robust_exp}
For this experiment, we aim to investigate the relation between model robustness and the distribution of pTU. By augmenting the training data, we can increase the robustness of the model. Therefore, we train five models on the \texttt{Fashion-MNIST} dataset: one without any data augmentation, and the others with varying levels of augmentation. We add four different levels of Gaussian noise to generate augmented data: from weak noise to strong noise, represented as $N(0.05, 1/9)$, $N(0.1, 1/9)$, $N(0.2, 1/9)$ and $N(0.3, 1/9)$. These training datasets are denoted as $\mathcal{D}_{\mathtt{F-MNIST}}^{05}$, $\mathcal{D}_{\mathtt{F-MNIST}}^{1}$, $\mathcal{D}_{\mathtt{F-MNIST}}^{2}$, and $\mathcal{D}_{\mathtt{F-MNIST}}^{3}$ respectively. We compare the statistical power (probability of rejecting $H_0$) to evaluate the robustness of model for $\mu=[0,0.4]$.

We apply the same method as in Section~\ref{sec:shift_exp} to distinguish between the distribution of pTU trained with \texttt{F-MNIST} and with augmented \texttt{F-MNIST}. The results are shown in Figure~\ref{fig:shift_power}. We observe that adding a soft noise during training helps improve the model's robustness. However, adding strong noise may have an adverse effect and reduce performance. Therefore, we can conclude that after adding suitable augmentation, the model learns the underlying essence of the dataset rather than the superficial perturbations.

% \begin{figure}[H]
% \begin{center}
% \includegraphics[width=0.9\textwidth]{Figs/robust_2x1.png}
% \end{center}
% \caption{}
% \label{fig:robust_power}
% \end{figure}

% \begin{figure}[H]
% \begin{center}
% \includegraphics[width=0.95\textwidth]{Figs/exp4_distibution.png}
% \includegraphics[width=0.95\textwidth]{Figs/exp41_distibution.png}
% \end{center}
% \caption{From left to right: the \textbf{pTU} distribution, generated by two models, for test sets with noise $\sigma$ ranging from $0, 0.1, \cdots, 0.5$, respectively.}
% \label{fig:ptu_aug_or_not}
% \end{figure}

% \begin{figure}[H]
% \begin{center}
% \includegraphics[trim=100 10 100 10, width=0.95\textwidth]{Figs/ill_noise.png}
% \end{center}
% \caption{From left to right: ratio of noise is $0, 0.05, \cdots, 0.25$}
% \label{fig:}
% \end{figure}

% \begin{figure}[H]
% \begin{center}
% \includegraphics[trim=100 0 100 0, width=0.95\textwidth]{Figs/robustness_with_aug.png}
% \includegraphics[trim=100 0 100 0, width=0.95\textwidth]{Figs/robustness_without_aug.png}
% \end{center}
% \caption{Upper: Training model with augmentation. From left to right: distribution of pTU, the p-values obtained from the permutation tests between the test set without noise and the test set with noise, the KL-divergence between the test set without noise and the test set with noise. Lower: Training model without augmentation.}
% \label{fig:}
% \end{figure}

\begin{figure}[ht]
\begin{center}
\includegraphics[width=0.45\textwidth]{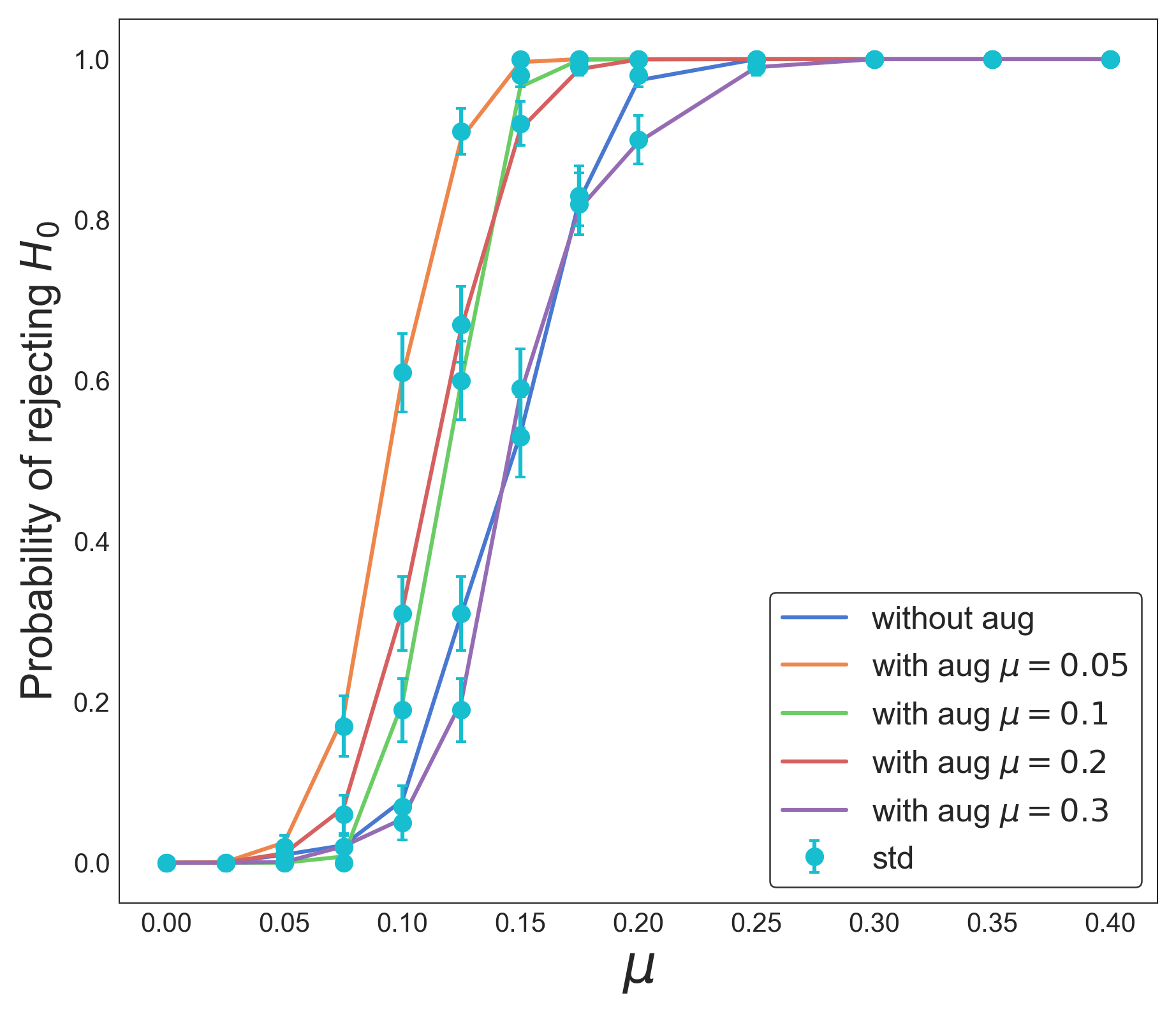}
\end{center}
\caption{Power curve for testing OOD. For a model trained by given training dataset $\mathcal{D}_{\texttt{F-MNIST}}^{\bullet}$, each curve stands for the probability of rejecting $H_0$ for $\mu=[0,0.4]$. Blue: the model is trained by non-augmentation dataset $\mathcal{D}_{\texttt{F-MNIST}}$. Orange, green, red, purple: the model is trained by augmentation dataset $\mathcal{D}_{\texttt{F-MNIST}}^{05}$, $\mathcal{D}_{\texttt{F-MNIST}}^{1}$, $\mathcal{D}_{\texttt{F-MNIST}}^{2}$, $\mathcal{D}_{\texttt{F-MNIST}}^{3}$, respectively.}
\label{fig:shift_power}
\end{figure}

\section{Proof of Theorem \ref{thm:stab_thm}}\label{sec:proof}

Before we prove Theorem \ref{thm:stab_thm}, we introduce some assumptions on network. Consider an $\ell$-th layer and the layer takes $x_{\ell-1}$ as input. Let $h_{\ell}$ and $h_{\ell-1}$ be number of nodes in the $\ell$-th layer and the $(\ell-1)$-th layer, respectively. Define the constant
\begin{equation}\label{eq:al}
A_{\ell}=\sup _{1\leq j\leq h_{\ell}} \sum_{i=1}^{h_{\ell-1}}\left|W_\ell(j,i)\right| .
\end{equation}
for $\ell=1,\cdots,L$. Then, if the activation map $\sigma_\ell$ is 1-Lipschitz, it is obvious that the transformation $f_\ell$ is $A_\ell$-Lipschitz in $\infty$-norm. On the other hand, we assume the transformation $f_\ell$ is bounded by $B_\ell$ in ($\infty$,$\infty$)-norm. That is, $\|f_\ell(x_{\ell-1})\|_\infty\leq B_\ell\|x_{\ell-1}\|_\infty$ for all $x_{\ell-1}\in\mathbb{R}^{h_{\ell-1}}$, $1\leq\ell\leq L$. Note that if $b_\ell=0$, then $B_\ell=A_\ell$.

Now, we can organize the following proposition.
\begin{proposition}\label{rem:fact}
For $x_{\ell-1},y_{\ell-1}\in\mathbb{R}^{h_{\ell-1}}$, there are some facts.
\begin{enumerate}
\item By Lipschitz property, we have $\|f_\ell(x_{\ell-1})-f_\ell(y_{\ell-1})\|_\infty<A_\ell\|x_{\ell-1}-y_{\ell-1}\|_\infty$. We further have
\begin{equation*}
\|f_\ell(x_{\ell-1})-f_\ell(y_{\ell-1})\|_\infty<\mathcal{A}_\ell\|x_0-y_0\|_\infty
\end{equation*}
where we denote $\mathcal{A}_\ell:=\left(\prod_{k=1}^\ell A_k\right)$.
\item By bounded transformation property, we have
\begin{equation*}
\|x_{\ell}\|_\infty=\|f_\ell(x_{\ell-1})\|_\infty\leq \mathcal{B}_\ell\|x_0\|_\infty
\end{equation*}
where denote $\mathcal{B}_\ell:=\left(\prod_{k=1}^\ell B_k\right)$\,.
% \item Let $\mathcal{C}_\ell:=\|W_\ell\|_{HS}$ to be Hilbert-Schmidt norm of matrix $W_{\ell}$. Then,
% \begin{equation*}
% \|D_\ell(x_{\ell-1},F_\theta)\|\leq \|W_\ell\odot x_{\ell-1}\|_{HS}\leq \mathcal{C}_\ell \|x_{\ell-1}\|_\infty\leq  \mathcal{C}_\ell \mathcal{B}_\ell\|x_{0}\|_\infty\,,
% \end{equation*}
% where $W_\ell\odot x_{\ell-1}=[W_\ell(j,i)x_{\ell-1}(i)]_{ji}$ is a $h_{\ell}\times h_{\ell-1}$ matrix.
\item Let $C_\ell:=\sup_{1\leq i\leq h_{\ell-1},1\leq j\leq h_\ell}|W_\ell(j,i)|$. Then, the persistent diagram is bounded by
\begin{equation*}
\|D_\ell(x_{\ell-1},F_\theta)\|_\infty\leq \mathcal{C}_\ell\mathcal{B}_{\ell-1}\|x_0\|_\infty\,.
\end{equation*}
\end{enumerate}
\end{proposition}

The stability of persistent diagram is stated as the following.
\begin{lemma}[Appendix A.2 in \cite{lacombe2021topological}]\label{lem:stab_pd}
Consider a $\ell$-th dense layer and layer inputs $x_\ell$, $y_\ell$. Assume that activation map $\sigma_\ell$ is 1-Lipschitz such that the $f_\ell$ is $A_\ell$-Lipschitz transformation. Then, we have
\begin{equation*}
W_{\infty}\left(D_{\ell}(x_{\ell-1}, F_\theta), D_{\ell}(y_{\ell-1}, F_\theta)\right) \leq A_{\ell}\|x_{\ell-1}-y_{\ell-1}\|_{\infty}
\end{equation*}
where $W_{\infty}(\cdot,\cdot)$ is $\infty$-Wasserstein distance (bottleneck distance).
\end{lemma}
Now we establish the stability of pTU stated by the following lemma.
\begin{lemma}\label{lem:lip}
Consider neural network $F_\theta$ and input $x,y\in\mathcal{D}_{\text{train}}$ which are bounded by $B_0$ in $\infty$-norm. Let activation maps $\sigma_\ell$ be 1-Lipschitz and the constants $\mathcal{A}_\ell,\mathcal{B}_\ell, C_\ell$ be defined in \ref{rem:fact}. Then, we have
\begin{equation}\label{eq:lip_ptu}
\left|\mathbf{pTU}(x| \mc{D}_{\mathrm{train}})-\mathbf{pTU}(y \mid \mc{D}_{\mathrm{train}})\right| \leq c\|x-y\|_\infty.
\end{equation}
where $c$ depends on $\mathcal{A}_\ell,\mathcal{B}_\ell, C_\ell$.
\end{lemma}
\begin{proof}
First, consider the $\ell$-th layer and the layer input $x_{\ell-1},y_{\ell-1}$. Let $N$ denote the number of nodes in the $\ell$-th bipartite graph, which is given by $h_\ell + h_{\ell-1} - 1 $. Thus, persistence diagram $D_\ell(x_{\ell-1},F_\theta)$ of the $\ell$-th bipartite graph is a ordered vector of size $N$. 

Let $m_{x_{\ell-1}}$ and $m_{y_{\ell-1}}$ denote the Frechet means of the random variables $D_\ell(x_{\ell-1}, F_\theta)$ and $D_\ell(y_{\ell-1}, F_\theta)$, respectively, with $m_{z}$ given by $\int_\Theta D_\ell(z, F_\theta) \pi(\theta \mid \mathcal{D}_\text{train}) \, d\theta$ for $z=x_{\ell-1}, y_{\ell-1}$. Note that the closed form of $p$-Wasserstein distance on $\mathbb{R}$ is well-known. Thus, we can rewrite Definition (\ref{eq:var_Dl}) as
\begin{align*}
&\var_{\theta}\left[D_{\ell}(x_{\ell-1}, F_{\theta}) \mid \mc{D}_{\text{train}}\right] = \inf_{\bar{m}} \int W_2(D_{\ell}(x_{\ell-1}, F_{\theta}), \bar{m})^2 \pi(\theta \mid \mc{D}_{\text{train}}) d\theta\\
=&\int W_2 (D_{\ell}(x_{\ell-1}, F_{\theta}), m_{x_{\ell-1}})^2 \pi(\theta \mid \mc{D}_{\text{train}}) d\theta=\int \frac{1}{N}\|D_{\ell}(x_{\ell-1}, F_{\theta})- m_{x_{\ell-1}}\|^2 \pi(\theta \mid \mc{D}_{\text{train}}) d\theta\,,
\end{align*}
where $\|\cdot\|$ is Euclidean 2-norm. Then, we have
\begin{align}
&N\left|\var_{\theta}\left[D_{\ell}(x_{\ell-1}, F_{\theta})\mid \mc{D}_{\text{train}}\right]-\var_{\theta}\left[D_{\ell}(y_{\ell-1}, F_{\theta})\mid \mc{D}_{\text{train}}\right]\right|\nonumber\\
&=\bigg|\int \|D_{\ell}(x_{\ell-1}, F_{\theta})- m_{x_{\ell-1}}\|^2 \pi(\theta \mid \mc{D}_{\text{train}}) d\theta\notag-\int \|D_{\ell}(y_{\ell-1}, F_{\theta})- m_{y_{\ell-1}}\|^2 \pi(\theta \mid \mc{D}_{\text{train}}) d\theta\bigg|\notag\\
&\begin{aligned}
&\leq\left|\int \left(D_{\ell}(x_{\ell-1}, F_{\theta})^\top D_{\ell}(x_{\ell-1}, F_{\theta})-D_{\ell}(y_{\ell-1}, F_{\theta})^\top D_{\ell}(y_{\ell-1}, F_{\theta})\right)\pi(\theta|\mathcal{D}_\text{train})d\theta\right|\\
&\hspace{1cm}+\left|m_{x_{\ell-1}}^\top m_{x_{\ell-1}}-m_{y_{\ell-1}}^\top m_{y_{\ell-1}}\right|
\end{aligned}\label{eq:diff_ptu}
\end{align}

For the first term in (\ref{eq:diff_ptu}), we have
\begin{align*}
&\left|\int \left(D_{\ell}(x_{\ell-1}, F_{\theta})^\top D_{\ell}(x_{\ell-1}, F_{\theta})-D_{\ell}(y_{\ell-1}, F_{\theta})^\top D_{\ell}(y_{\ell-1}, F_{\theta})\right)\pi(\theta|\mathcal{D}_\text{train})d\theta\right|\\
\leq &\int \big|\|D_{\ell}(x_{\ell-1}, F_{\theta})\|-\|D_{\ell}(y_{\ell-1}, F_{\theta})\|\big|\left(\|D_{\ell}(x_{\ell-1}, F_{\theta})\|+\|D_{\ell}(y_{\ell-1}, F_{\theta})\|\right)\pi(\theta|\mathcal{D}_\text{train})d\theta\\
\leq &\int \|D_{\ell}(x_{\ell-1}, F_{\theta})-D_{\ell}(y_{\ell-1}, F_{\theta})\|\left(\|D_{\ell}(x_{\ell-1}, F_{\theta})\|+\|D_{\ell}(y_{\ell-1}, F_{\theta})\|\right)\pi(\theta|\mathcal{D}_\text{train})d\theta\\
\leq &N\int W_\infty\left(D_{\ell}(x_{\ell-1}, F_{\theta}),\,D_{\ell}(y_{\ell-1}, F_{\theta})\right)\left(C_\ell\mathcal{B}_{\ell-1}\|x_0\|_\infty+C_\ell\mathcal{B}_{\ell-1}\|y_0\|_\infty\right)\pi(\theta|\mathcal{D}_\text{train})d\theta\\
\leq &N\int \mathcal{A}_\ell\left\|x_{\ell-1}-y_{\ell-1}\right\|_\infty\left(C_\ell\mathcal{B}_{\ell-1}\|x_0\|_\infty+(C_\ell\mathcal{B}_{\ell-1}\|y_0\|_\infty\right)\pi(\theta|\mathcal{D}_\text{train})d\theta\\
\leq &2NC_\ell\mathcal{A}_\ell\mathcal{B}_{\ell-1}B_0\left\|x_{\ell-1}-y_{\ell-1}\right\|_\infty
\end{align*}
where the third inequality is due to Proposition \ref{rem:fact} and the last inequality is due to Lemma \ref{lem:stab_pd} and the boundness of $x_0$, $y_0$.

Second, for the second term in (\ref{eq:diff_ptu}), we have
\begin{align*}
&\left|m_{x_{\ell-1}}^\top m_{x_{\ell-1}}-m_{y_{\ell-1}}^\top m_{y_{\ell-1}} \right|\leq\big|\|m_{x_{\ell-1}}\|-\|m_{y_{\ell-1}}\|\big|\left(\|m_{x_{\ell-1}}\|+\|m_{y_{\ell-1}}\|\right)\\
\leq & \int \left\|D_\ell(x_{\ell-1},F_\theta)-D_\ell(y_{\ell-1},F_\theta)\right\|\pi(\theta|\mathcal{D}_{\text{train}})d\theta\left(C_\ell\mathcal{B}_{\ell-1}\|x_0\|_\infty+C_\ell\mathcal{B}_{\ell-1}\|y_0\|_\infty\right)\\
% \leq &N^2L^2\left\|x-y\right\|_\infty\left(\|x\|_\infty+\|y\|_\infty\right)
\leq &2NC_\ell\mathcal{A}_\ell\mathcal{B}_{\ell-1}B_0\left\|x_{\ell-1}-y_{\ell-1}\right\|_\infty
\end{align*}
Now, for $\ell$-th layer, we have $(\ref{eq:diff_ptu})\leq 4NC_\ell\mathcal{A}_\ell\mathcal{B}_{\ell-1}B_0\left\|x_{\ell-1}-y_{\ell-1}\right\|_\infty$. By definition of \textbf{pTU} in (\ref{eq:ptu}), for $x,y\in\mathcal{D}_{\text{train}}$, we have
\begin{align*}
&|\mathbf{pTU}(x)-\mathbf{pTU}(y)|\leq \frac{1}{L}\sum^{L}_{\ell=1}\bigg|\var_{\theta}\left[D_{\ell}(x_{\ell-1}, F_{\theta})\mid \mc{D}_{\text{train}}\right]-\var_{\theta}\left[D_{\ell}(y_{\ell-1}, F_{\theta})\mid \mc{D}_{\text{train}}\right]\bigg|\\
&\quad\quad\leq \sum^{L}_{\ell=1}\frac{4C_\ell\mathcal{A}_\ell\mathcal{B}_{\ell-1}B_0}{L}\left\|x_{\ell-1}-y_{\ell-1}\right\|_\infty\\
&\quad\quad\leq \sum^{L}_{\ell=1}\frac{4C_\ell\mathcal{A}_\ell\mathcal{B}_{\ell-1}B_0}{L}\prod^{\ell-1}_{k=1}A_k\|x_0-y_0\|_\infty\leq\sum^{L}_{\ell=1}\frac{4C_\ell\mathcal{A}_\ell\mathcal{A}_{\ell-1}\mathcal{B}_{\ell-1}B_0}{L}\|x-y\|_\infty\,.
\end{align*}
Hence, the constant $c$, defined as $c := \sum_{\ell=1}^{L} \frac{4C_\ell \mathcal{A}_\ell \mathcal{B}_{\ell-1} B_0}{L}$, is the desired constant in \ref{eq:lip_ptu}.
\end{proof}

We restate Theorem \ref{thm:stab_thm} here for convenience.
\begin{theorem}
Consider neural network $F_\theta$. Let random variables $X,Y\sim P(X)$ are bounded. Let activation maps $\sigma_\ell$ be 1-Lipschitz and the constants $\mathcal{A}_\ell,\mathcal{B}_\ell, C_\ell$ be defined in Proposition \ref{rem:fact}. Then, under the normality assumption of posterior, we have
\begin{equation}
W_2(\mathbf{pTU}(X|\mathcal{D}_{\text{train}}), \mathbf{pTU}(Y|\mathcal{D}_{\text{train}}))\leq CW_2(X,Y)
\end{equation}
where $C$ depends on $\mathcal{A}_\ell,\mathcal{B}_\ell, C_\ell$.
\end{theorem}

\begin{proof}
By definition, it is equivalent to show
\begin{equation*}
\inf_{\gamma\in \Gamma}\mathbb{E}_{\gamma}\left[\left|\mathbf{pTU}(X| \mc{D}_{\mathrm{train}})-\mathbf{pTU}(Y \mid \mc{D}_{\mathrm{train}})\right|^2\right]\leq c \inf_{\gamma\in \Gamma}\mathbb{E}_{\gamma}\left[\left\|X-Y\right\|_\infty^2\right]\,,
\end{equation*}
where $\Gamma$ is the set of all coupling of $\mathcal{D}_{\text{train}}\times \mathcal{D}_{\text{train}}$. Let $\epsilon>0$. There exists $\bar{\gamma}\in\Gamma$ such that
\begin{equation*}
\mathbb{E}_{\bar{\gamma}}\left[\left\|X-Y\right\|_\infty^2\right]\leq \inf_{\gamma\in \Gamma}\mathbb{E}_{\gamma}\left[\left\|X-Y\right\|_\infty^2\right]+\epsilon\,.
\end{equation*}
By Lemma \ref{lem:lip}, for given $\bar{\gamma}$, we can obtain that
\begin{equation*}
\mathbb{E}_{\bar{\gamma}}\left[\left|\mathbf{pTU}(X| \mc{D}_{\mathrm{train}})-\mathbf{pTU}(Y \mid \mc{D}_{\mathrm{train}})\right|^2\right] \leq c^2\mathbb{E}_{\bar{\gamma}}\left[\|X-Y\|_\infty^2\right].
\end{equation*}
Thus, by putting above two inequalities together, we have
\begin{equation*}
  \mathbb{E}_{\bar{\gamma}}\left[\left|\mathbf{pTU}(X| \mc{D}_{\mathrm{train}})-\mathbf{pTU}(Y \mid \mc{D}_{\mathrm{train}})\right|^2\right] \leq c^2\inf_{\gamma\in \Gamma}\mathbb{E}_{\gamma}\left[\left\|X-Y\right\|_\infty^2\right]+\epsilon\,.
\end{equation*}
On the other hand, to show the existence of $\bar{\gamma}$
\begin{equation*}
\inf_{\gamma\in \Gamma}\mathbb{E}_{\gamma}\left[\left|\mathbf{pTU}(X| \mc{D}_{\mathrm{train}})-\mathbf{pTU}(Y \mid \mc{D}_{\mathrm{train}})\right|^2\right]\leq c^2 \mathbb{E}_{\bar{\gamma}}\left[\left\|X-Y\right\|_\infty^2\right]
\end{equation*}
is straightforward, refer to Lemma \ref{lem:lip}. Therefore, the desired result follows.
\end{proof}
\section{Conclusion}\label{sec:conclusion}
In this work, we generalize the concept of Topological Uncertainty to Bayesian neural networks, leveraging activation graphs to analyze the behavior of the model. Additionally, we show that pTU exhibits a stability property, which allows for distinguishing whether a given dataset is out-of-distribution (OOD). This property not only enhances the understanding of model predictions but also provides a robust indicator for assessing the reliability of the network under varying conditions. Our findings highlight the potential of pTU as a powerful tool for evaluating both OOD and model robustness.

% Open question:
% \begin{itemize}
%     \item pTU of training data $\leq$ pTU of test data (stochastically less than)
% \end{itemize}

\appendix

\section{Additional Experimental Details}
\subsection{Networks architectures and training}
\paragraph{Architecture} In Section~\ref{sec:experiment}, all networks are implemented as multi-layer perceptrons (MLPs), which input the flattened image as 784-dim input and return probabilities of 10 class. In Section \ref{sec:arch}, the network architectures are as follows: M3 ($784 \to 80 \to 10$), M4 ($784 \to 64 \to 16 \to 10$), M5 ($784 \to 40 \to 20 \to 20 \to 10$), M6 ($784 \to 32 \to 16 \to 16 \to 16 \to 10$), M7 ($784 \to 20 \to 20 \to 16\to 12\to 12 \to 10$), M8 ($784 \to 16\to 16 \to 12\to 12\to 12\to 12 \to 10$), M9 ($784 \to 16\to 12\to 12\to 10\to 10\to 10\to 10 $). All the other models in Section~\ref{sec:ood_exp} - Section~\ref{sec:robust_exp} are 3-layer models with the architecture ($784 \to 64 \to 16 \to 10$).

\paragraph{Loss} Variational learning determines the parameters $\phi$ of a weight distribution $q(\theta|\phi)$ in order to minimize the KL divergence from the actual Bayesian posterior distribution $\pi(\theta|\mathcal{D})$ on weights. Let $\pi(\theta)$ be the prior distribution of the parameters and $P(\mathcal{D}\mid \theta)$ be the likelihood, which corresponds to the cross-entropy or softmax loss. Then, according to \cite{blundell2015weight}, the loss function is designed as
\begin{align}
\mathcal{F}(\mathcal{D}, \theta)=\mathrm{KL}[q(\theta \mid \phi) \| & \pi(\theta)]-\mathbb{E}_{q(\theta \mid \phi)}[\log P(\mathcal{D} \mid \theta)].
\end{align}

\paragraph{Training} We utilized the Adam optimizer for its efficient computation. We train the model without regularization, following the guidelines in \cite{blundell2015weight}. To prevent overfitting and ensure efficient training, we implemented early stopping. We specifically set the patience to 25 epochs. The test accuracies of various architectural models trained on \texttt{F-MNIST} and discussed in Section \ref{sec:arch} are shown in Table \ref{tab:acc_arch}. The test accuracies of 3-layer model trained on various datasets are listed in Table \ref{tab:acc_set}.

\begin{table}[ht]
\caption{The accuracies of models trained on \texttt{F-MNIST}, which is discussed in Section \ref{sec:arch}.}
\begin{center}
\begin{tabular}
{
>{\centering\arraybackslash}m{2.5cm} ||
>{\centering\arraybackslash}m{2.2cm}
>{\centering\arraybackslash}m{2.2cm}
>{\centering\arraybackslash}m{2.5cm}
>{\centering\arraybackslash}m{2.5cm}
}
\noalign{\hrule height 1.5pt}
Model name  & \# of layers & \# of neurals & Test acc \\
\noalign{\hrule height 1.5pt}
M3 & 2 & 874 & 87.56 \%\\
M4 & 3 & 874 & 87.33 \%\\
M5 & 4 & 874 & 86.90 \%\\
M6 & 5 & 874& 86.98 \%\\
M7 & 6 & 874& 86.1 \%\\
M8 & 7 & 874& 85.12 \%\\
M9 & 8 & 874& 83.24 \%\\
\noalign{\hrule height 1.5pt}
\end{tabular}
\end{center}
\label{tab:acc_arch}
\end{table}

\begin{table}[ht]
\caption{The test accuracies of 3-layer model trained on various datasets.}
\begin{center}
\begin{tabular}
{
>{\centering\arraybackslash}m{4.5cm} ||
>{\centering\arraybackslash}m{2.2cm}
}
\noalign{\hrule height 1.5pt}
Training set  & Test acc \\
\noalign{\hrule height 1.5pt}
\texttt{SVHN w/o aug} & 70.36 \%\\
\texttt{MNIST w/o aug} & 97.39 \%\\
\texttt{F-MNIST w/o aug} & 87.33 \%\\
\texttt{F-MNIST w/ aug $\mu=0.05$} & 88.26 \%\\
\texttt{F-MNIST w/ aug $\mu=0.1$} & 88.39 \%\\
\texttt{F-MNIST w/ aug $\mu=0.2$} & 87.24 \%\\
\texttt{F-MNIST w/ aug $\mu=0.3$} & 85.31 \%\\
\noalign{\hrule height 1.5pt}
\end{tabular}
\end{center}
\label{tab:acc_set}
\end{table}

\subsection{Datasets preprocessing}
\paragraph{Datasets description} The \texttt{MNIST} and \texttt{F-MNIST} datasets \cite{LeCun2005, Xiao2017} both include 60,000 training sample and 10,000 test samples. These samples are grayscale images with a resolution of $28\times 28$ pixels, categorized into 10 classes. \texttt{MNIST} represents handwritten digits, while \texttt{F-MNIST} contains images of fashion items.

The \texttt{SVHN} dataset \cite{Netzer2011} consists of 73,257 training sample and 26,032 test samples. These samples are RGB images with dimensions $32 \times 32 \times 3$, representing street view of house numbers. The images are divided into 10 classes, each corresponding to a digit extracted from natural images.

The \texttt{CIFAR10} dataset \cite{Krizhevsky2009} consists of 10,000 test samples. These samples are RGB images with dimensions $32 \times 32 \times 3$, representing nature images. Since the \texttt{CIFAR10} is used here, the images are divided into 10 classes.

\paragraph{Preprocessing} All image pixels are rescaled to [0, 1] from [0, 255], then converted to grayscale and resized to $28\times 28\times 1$. Additionally, images with added noise are clipped to the range [0, 1].

\section{Additional Results}
\subsection{More figures for Section~\ref{sec:ood_exp}}
Figure~\ref{fig:ood_distr} shows the distributions of $\mathbf{pTU}(X|\mathcal{D}_{\text{train}})$ for given training dataset $\mathcal{D}_{\text{train}}$ and given underlying distribution of dataset $X'\sim P(X')$ where $P(X')$ may or may not be the same as the underlying distribution of the training dataset, as discussed in Section~\ref{sec:ood_exp}. Each subfigure presents a model trained on a given dataset $\mathcal{D}_{\text{train}}$ where $\mathcal{D}_{\text{train}}$ may be \texttt{SVHN}, \texttt{MNIST} or \texttt{F-MNIST}. Furthermore, each curve stands for an empirical distribution $\{\widehat{pTU}(x_i|\mathcal{D}_{\text{train}})\}_{x_i\in\mathcal{X}'}$ where $\mathcal{X}'$ may be drawn from the underlying distribution of \texttt{SVHN}, \texttt{MNIST}, \texttt{F-MNIST} or \texttt{CIFAR}. In particular, we highlight the distributions with blue dashed lines when $P(X') = P(X)$, indicating that $P(X')$ represents the underlying distribution of the training set. It is important to note that these represent the same underlying distributions, but the data we sample from them may not necessarily be the same.
\begin{figure}[ht]
\begin{center}
\includegraphics[width=0.95\textwidth]{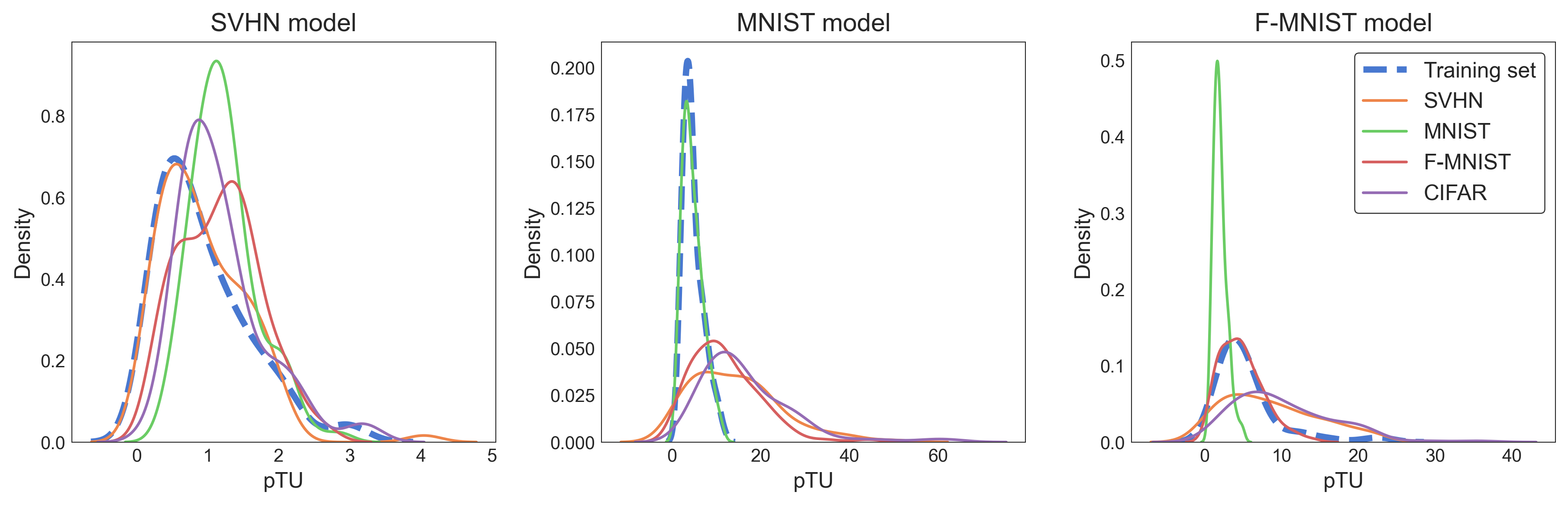}
\end{center}
\caption{Left: the empirical distributions $\{\widehat{pTU}(x_i|\mathcal{D}_{\texttt{SVHN}})\}_{x_i\in\mathcal{X}'}$ where $\mathcal{X}'$ drawn from different underlying distribution. Middle: the empirical distributions $\{\widehat{pTU}(x_i|\mathcal{D}_{\texttt{MNIST}})\}_{x_i\in\mathcal{X}'}$ where $\mathcal{X}'$ drawn from different underlying distribution. Right: the empirical distributions $\{\widehat{pTU}(x_i|\mathcal{D}_{\texttt{F-MNIST}})\}_{x_i\in\mathcal{X}'}$ where $\mathcal{X}'$ drawn from different underlying distribution.}
\label{fig:ood_distr}
\end{figure}

\subsection{More figures for Section \ref{sec:robust_exp}}
Figure~\ref{fig:robust_distri} shows the distribution of $\mathbf{pTU}(s_\mu(X)|\mathcal{D}_{\texttt{F-MNIST}}^{\bullet})$ for given $\mu$ and given training dataset $\mathcal{D}_{\texttt{F-MNIST}}^{\bullet}$, which represents the distribution of the results from a single experiment discussed in Section \ref{sec:robust_exp}. Each subfigures presents the shift level $s_\mu$ of the datasets and each curve stands for different training dataset $\mathcal{D}_{\texttt{F-MNIST}}^\bullet$.

%such as $\mathcal{D}_{\texttt{F-MNIST}}$, $\mathcal{D}_{\texttt{F-MNIST}}^{05}$, $\mathcal{D}_{\texttt{F-MNIST}}^{2}$, $\mathcal{D}_{\texttt{F-MNIST}}^{3}$. 

\begin{figure}[ht]
\begin{center}
\includegraphics[width=0.95\textwidth]{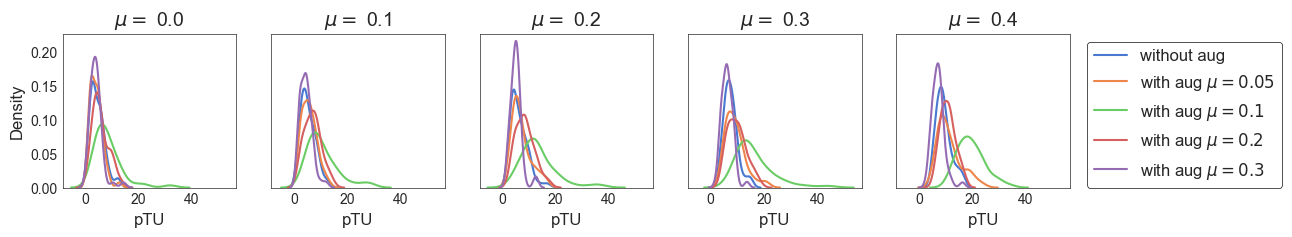}
\end{center}
\caption{Leftmost: five curves show the empirical distributions $\{\widehat{pTU}(s_0(x_i)|\mathcal{D}_{\texttt{F-MNIST}})\}$, $\{\widehat{pTU}(s_0(x_i)|\mathcal{D}_{\texttt{F-MNIST}}^{05})\}$, $\{\widehat{pTU}(s_0(x_i)|\mathcal{D}_{\texttt{F-MNIST}}^{1})\}$, $\{\widehat{pTU}(s_0(x_i)|\mathcal{D}_{\texttt{F-MNIST}}^{2})\}$ and $\{\widehat{pTU}(s_0(x_i)|\mathcal{D}_{\texttt{F-MNIST}}^{3})\}$, respectively. The other subfigures are similar, with the only difference being the shift level, meaning that $s_0$ should be changed to $s_{0.1}$, $s_{0.2}$, $s_{0.3}$ and $s_{0.4}$, respectively.}
\label{fig:robust_distri}
\end{figure}

\bibliography{reference}

\end{document}